\documentclass[twoside]{article}

\usepackage[accepted]{aistats2025}
\usepackage[round]{natbib}

\usepackage{amsmath,amsfonts,bm}

\def\eqref#1{equation~\ref{#1}}

\def\1{\bm{1}}

\DeclareMathAlphabet{\mathsfit}{\encodingdefault}{\sfdefault}{m}{sl}
\SetMathAlphabet{\mathsfit}{bold}{\encodingdefault}{\sfdefault}{bx}{n}

\def\gD{{\mathcal{D}}}

\def\gL{{\mathcal{L}}}

\newcommand{\R}{\mathbb{R}}

\newcommand{\sigmoid}{\sigma}

\newcommand{\KL}{D_{\mathrm{KL}}}

\usepackage{hyperref}
\hypersetup{breaklinks=true,colorlinks,urlcolor={red!66!black},linkcolor={red!66!black},citecolor={blue!66!black}}

\usepackage{url}
\usepackage{booktabs}       
\usepackage{graphicx}
\usepackage{amsthm,amsmath,amsfonts,amssymb}
\allowdisplaybreaks[1]
\usepackage{subcaption}
\usepackage{cleveref}
\usepackage{cancel}
\usepackage{pifont}
\usepackage{multirow}
\usepackage{paralist}
\usepackage{float}
\usepackage{xcolor}
\usepackage{xspace} 

\usepackage[T1]{fontenc}
\usepackage[ansinew]{inputenc}
\usepackage{lmodern}
\usepackage{relsize}
\usepackage{anyfontsize}
\usepackage{textcomp}

\usepackage{lipsum}
\usepackage{wrapfig}

\crefname{equation}{Eq.}{Eqs.}

\def\methodabbrv{$f$-PO\xspace}

\newcommand{\RKLDiv}{\mathbb{D}_{\textup{RKL}}}
\newcommand{\KLDiv}{\mathbb{D}_{\textup{KL}}}
\newcommand{\FDiv}{\mathbb{D}_f}
\newcommand{\expect}{\mathbb{E}}
\newcommand{\pref}{\pi_{\textnormal{ref}}}
\newcommand{\data}{\mathcal{D}}

\newtheorem{theorem}{Theorem}

\newtheorem{lemma}{Lemma}

\begin{document}

%
\runningtitle{$f$-PO: Generalizing Preference Optimization with $f$-divergence Minimization}

%

\twocolumn[

\aistatstitle{$f$-PO: Generalizing Preference Optimization with\\ $f$-divergence Minimization}

\aistatsauthor{Jiaqi Han$^*$, Mingjian Jiang$^*$,  Yuxuan Song, Stefano Ermon, Minkai Xu$^{*\dagger}$}

\centering{$^*$Equal contribution $^\dagger$Corresponding author}

\centering{\texttt{\{jiaqihan,jiangm,minkai\}@cs.stanford.edu}}

\aistatsaddress{Computer Science, Stanford University} ]

\begin{abstract}
\looseness=-1
Preference optimization has made significant progress recently, with numerous methods developed to align language models with human preferences. This paper introduces $f$-divergence Preference Optimization (\methodabbrv), a novel framework that generalizes and extends existing approaches. \methodabbrv minimizes $f$-divergences between the optimized policy and the optimal policy, encompassing a broad family of alignment methods using various divergences. Our approach unifies previous algorithms like DPO and EXO, while offering new variants through different choices of $f$-divergences. We provide theoretical analysis of \methodabbrv's properties and conduct extensive experiments on state-of-the-art language models using benchmark datasets. Results demonstrate $f$-PO's effectiveness across various tasks, achieving superior performance compared to existing methods on popular benchmarks such as AlpacaEval 2, Arena-Hard, MT-Bench, and Open LLM Leaderboard v2. Additionally, we present ablation studies exploring the impact of different $f$-divergences, offering insights into the trade-offs between regularization and performance in offline preference optimization. Our work contributes both practical algorithms and theoretical understanding to the field of language model alignment. Code is available at \url{https://github.com/MinkaiXu/fPO}.
\end{abstract}

\section{Introduction}

Despite the impressive emerging ability of modern large language models (LLMs) by weak-supervised learning, further aligning these models with human feedback is still crucial~\citep{leike2018scalable}, ensuring the LLMs are more helpful, honest~\citep{askell2021general}, harmless~\citep{bai2022traininghelpfulharmlessassistant}, faithful~\citep{ji2023survey}, and unbiased~\citep{bender2021dangers}. Reinforcement learning from human feedback (RLHF)~\citep{christiano2017deep,stiennon2020learning,ouyang2022training} is the canonical
paradigm for fine-tuning LLMs towards effective alignment. It typically consists of several separate procedures, including training a reward model to capture the human values with well-labeled preference datasets, and optimizing the LLM as the policy model to maximize the reward. While these approaches have achieved remarkable results, they present notable optimization challenges due to the multi-stage process.

Lately, to mitigate this training instability and complexity, simpler offline methods such as Direct Prefernece Optimization (DPO)~\citep{rafailov2023direct} have been attracting increasing attention~\citep{Azar2023AGT,Zhao2023SLiCHFSL}. Instead of a multi-stage RL pipeline, these methods propose to directly align LLMs with pairwise comparison datasets, which avoids the additional efforts for training the reward model. In DPO, the reward function is instead parameterized as the log density ratio between the policy model and a fixed reference model (typically the one after supervised fine-tuning). DPO enjoys efficient and stable optimization by training the log density ratio with a binary classification objective, following the Bradley-Terry model~\citep{bradley1952rank}. Afterward, numerous methods have been proposed to optimize the log density ratio under various different convex functions $f$ such as IPO~\citep{Azar2023AGT} and EXO~\citep{ji2024towards}, and GPO~\citep{tang2024generalized} further provides a unified view of the existing algorithms. However, the parameterization of the reward function to density ratio is not always guaranteed when the policy is suboptimal, and the specific choice of convex function $f$ lacks theoretical guidance and remains heuristic. 

In this paper, we argue that a more principled framework can be developed from a general distribution-matching perspective, which can also provide direct insight in choosing the convex function $f$ based on desirable distribution behavior. To this end, we propose \methodabbrv, a general and principled family of preference optimization algorithms. The key innovation of our framework is to formulate preference optimization as a distribution-matching problem between the policy model and the underlying optimal model via $f$-divergence minimization. Such formulation induces a general alignment objective by wrapping the density ratio of two policies in various $f$ functions satisfying certain requirements. Importantly, our framework recovers DPO and EXO with reverse and forward KL divergences, and generalizes to other cases with arbitrary $f$-divergences. More concretely, we make the following contributions over state-of-the-art preference optimization algorithms:
\begin{itemize}
    \item We derive generalized offline alignment objectives from a distribution-matching perspective using $f$-divergences, and cover several major previous methods as special cases under the pairwise comparison datasets setting.
    \item We introduce new algorithms not yet in the current literature by using novel $f$-divergences, such as $\alpha$-divergence and Jeffrey's divergence.
    \item We provide principled ways to further combine \methodabbrv with other state-of-the-art preference optimization methods such as SimPO~\citep{meng2024simpo}, by alternative the inner density ratio with other approximations. 
    \item We conduct detailed analysis with different $f$-divergences, and observe reasonable performance trade-offs across \methodabbrv variants following corresponding divergence characteristics.
\end{itemize}
We conduct comprehensive experiments to compare \methodabbrv with competitive existing alignment methods. Notably, when using $\alpha$-divergence, our approach can consistently achieve superior or comparable performance against state-of-the-art algorithms on standard benchmarks, with up to $45.3\%$ length-controlled winning rate on AlpacaEval 2 by finetuning on Llama3-8B-instruct. The results demonstrate that our approach leads to substantial improvements with the principled $f$-divergence design space.

\section{Related Work}

Aligning pretrained large language models with human preferences for high-quality responses is vital in natural language generation. RLHF \citep{christiano2017deep,ouyang2022training} has emerged as a common solution to address this problem by utilizing popular RL methods such as Proximal Policy Optimization (PPO) \citep{schulman2017proximalpolicyoptimizationalgorithms}. However, this approach faces limitations due to training instability and the complexity introduced by its two-stage pipeline.
DPO~\citep{rafailov2023direct} overcomes this limitation by defining the preference loss as a function of the policy directly given the pairwise preference data.  Followup works have extended this algorithm to utilize multiple ranked responses instead of pairwise preference data \citep{yuan2023rrhfrankresponsesalign, liu2024lipolistwisepreferenceoptimization, song2024preferencerankingoptimizationhuman}, and avoid dependence on the reference model, effectively merging the instruction tuning phase and preference optimization phase \citep{Hong2024ORPOMP, meng2024simpo}. 
GPO~\citep{tang2024generalized} unifies several DPO variants~\citep{Azar2023AGT,liu2024statisticalrejectionsamplingimproves} into a family of algorithms and provides an empirical analysis of the performance trade-off, but leave the theoretical analysis in a heuristic manner. Notably, \citet{wang2023beyond} also introduces $f$-divergence into the alignment problem, but the divergences are only applied to the regularization term of the RL formulation.
Furthermore, all the above works are built on the reparameterization of reward function by policy likelihood, which cannot be guaranteed in practice~\citep{ji2024towards}.
In contrast, our method unifies the algorithm based on the theoretical framework of statistical distribution divergence~\citep{csiszar2004information,liese2006divergences}, which offers a principled to choose the specific algorithm based on divergence characteristics.
EXO~\citep{ji2024towards} proposes learning the policy by exactly optimizing the RLHF objective that minimizes the reverse KL divergence against the optimal policy, which is a special case of our proposed framework.~\citet{xiao2024leverage} leverages demonstration data for alignment using self-imitation learning. In addition, numerous offline preference optimization works propose different training objectives to emphasize various behaviors~\citep{yuan2023rrhf, Zhao2023SLiCHFSL, xu2024contrastive, Ethayarajh2024KTOMA}. 


\begin{figure*}[!t]
    \centering
    \begin{subfigure}[b]{0.245\textwidth}
        \includegraphics[width=\textwidth]{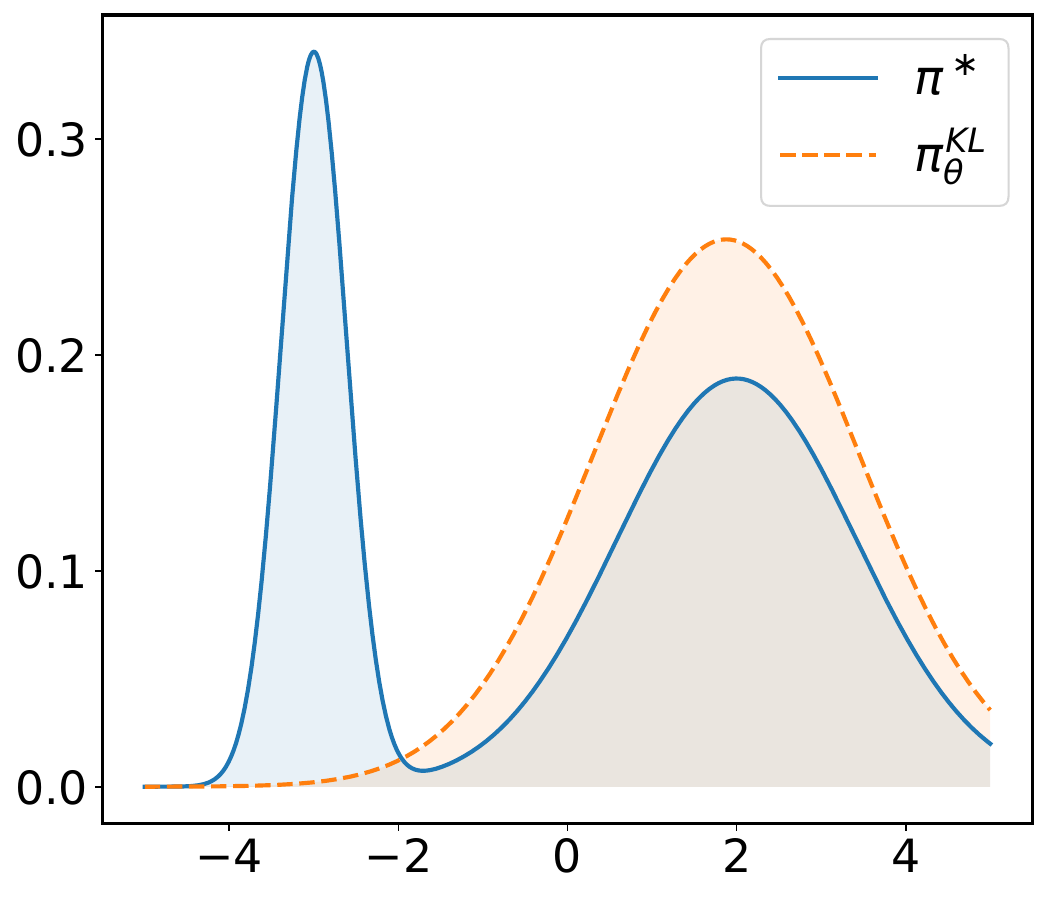}
        \caption{KL Divergence}
    \end{subfigure}
    \begin{subfigure}[b]{0.245\textwidth}
        \includegraphics[width=\textwidth]{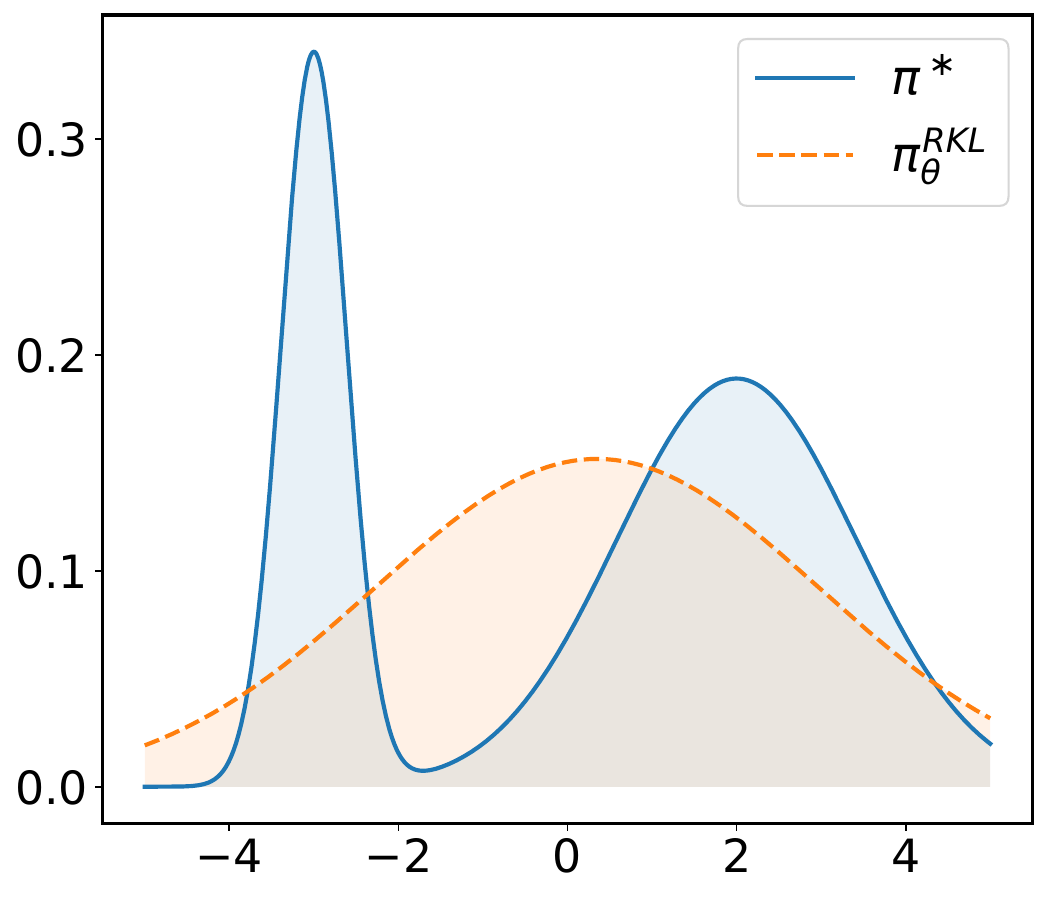}
        \caption{Reverse KL Divergence}
    \end{subfigure}
    \begin{subfigure}[b]{0.245\textwidth}
        \includegraphics[width=\textwidth]{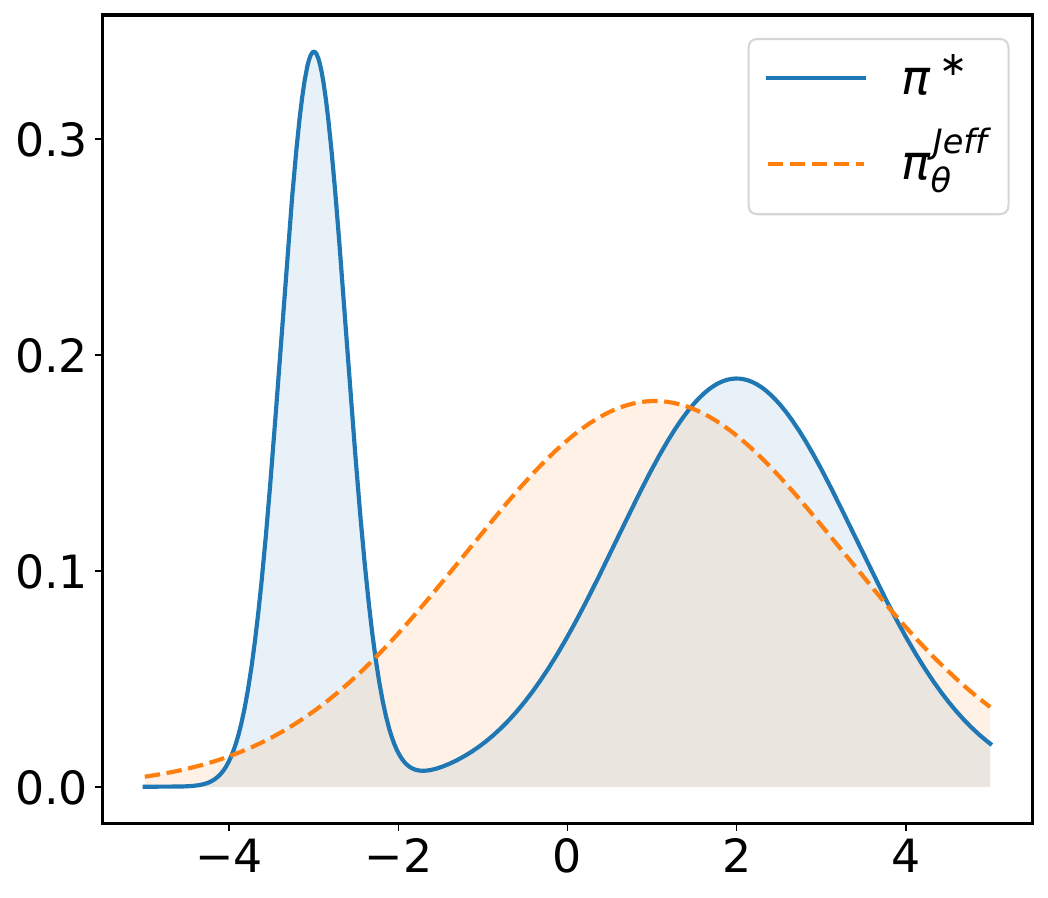}
        \caption{Jeffrey's Divergence}
    \end{subfigure}
    \begin{subfigure}[b]{0.245\textwidth}
        \includegraphics[width=\textwidth]{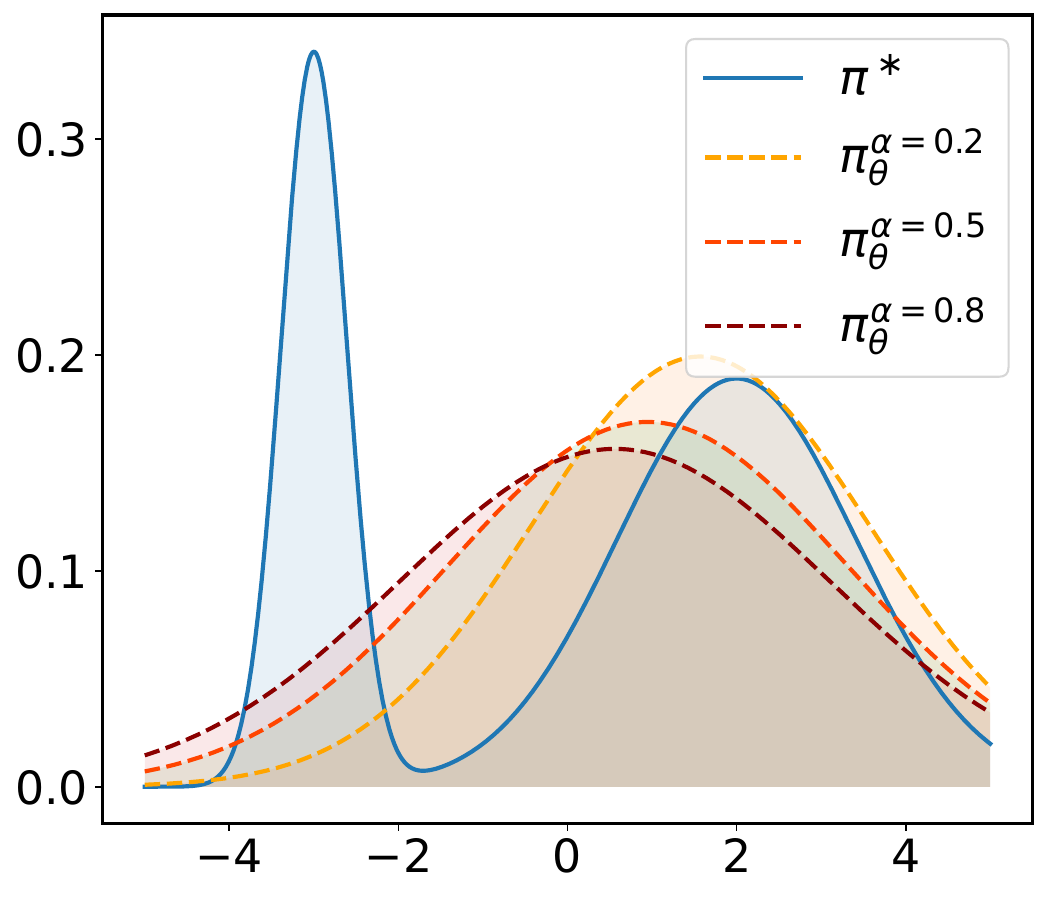}
        \caption{$\alpha$-Divergence}
    \end{subfigure}
    \caption[]{Illustration of the behavior of several divergences $\FDiv(\pi_\theta\| \pi^*)${\footnotemark} in the $f$-divergence family. $\pi^\ast$ indicates the target policy while $\pi_\theta$ is the parameterized policy obtained by optimizing $\theta$ through minimizing certain $f$-divergence. Notably, $\alpha$-divergence exhibits varying characteristics when adjusting the values of $\alpha$.}
    \label{fig:six-figures}
\end{figure*}

\section{Preliminaries}
\label{sec:prelim}

\subsection{Reinforcement Learning from Human Feedback}
\label{subsec:prelim-rlhf}

We introduce the general Reinforcement Learning from Human Feedback (RLHF), an offline fine-tuning framework after the base model has been well pretrained. A typical RLHF process is composed of three steps: 1) supervised fine-tuning, 2) reward model learning, and 3) RL fine-tuning. In the RL fine-tuning process, given the dataset $\data$ consisting of prompts and reward function $r_{\phi}(\cdot|x)$ learned following the Bradley-Terry model~\citep{bradley1952rank} on the preference dataset, the language model $\pi_\theta$ is optimized by maximizing the following objective:
\begin{equation}
\label{eq:rlhf}
    \expect_{x \sim \data, y\sim \pi_{\theta}(\cdot|x)} \left[r_{\phi}(x,y)\right] - \beta \KL(\pi_{\theta}(\cdot|x)|\pref(\cdot|x)),
\end{equation}
where $\pref(\cdot|x)$ is the fixed reference model obtained after supervised fine-tuning, and $\beta$ is a coefficient controlling the reverse KL divergence penalty. Analytically, the solution to the regularized objective above can be written as~\citep{peters2007reinforcement}:
\begin{equation}
\label{eq:pi-optimal}
    \pi^* = \frac{1}{Z(x)} \pref \exp{\left(\beta^{-1}r_\phi(x,y)\right)},
\end{equation}
where $Z(x)$ is the partition function. The RL fine-tuning aims to learn a parameterized policy $\pi_\theta$ to approximate the optimal policy $\pi^*$. 

\subsection{Preference Optimization}
\label{subsec:prelim-po}

The DPO method~\citep{rafailov2023direct} is one of the most popular offline preference optimization methods. It uses the RL objective under the reverse KL divergence constraint \Cref{eq:rlhf} to build a functional mapping between the reward model and the optimal policy:
\begin{equation}
    r(\cdot|x) = \beta \log \frac{\pi_{\theta}(\cdot|x)}{\pref(\cdot|x)} + {\beta\log Z(x)},
\end{equation}
which allows the direct optimization of the policy $\pi_\theta$. This is obtained by plugging the reward into the Bradley-Terry model and reparameterizing the reward function using the policy \textit{i.e.}, the language model, in a supervised manner:
\begin{equation}
\begin{aligned}
\label{eq:dpo-obj}
    & -\expect_{(x, y_w, y_l)\sim \data} \Big[\log \sigma \Big(\beta \log \frac{\pi_{\theta}(y_w|x)}{\pref(y_w|x)} + \cancel{\beta\log Z(x)} \\
    & \quad \quad \quad \quad \quad \quad - \beta \log \frac{\pi_{\theta}(y_l|x)}{\pref(y_l|x)} - \cancel{\beta\log Z(x)} \Big)\Big],
\end{aligned}
\end{equation}
where $\gD = (\{x, y^w, y^l\})$ is the dataset of ranked pairs of generations and their prompts, where $y^w$ and $y^l$ denote ``winning'' and ``losing'' samples, and $\sigma$ is the sigmoid function, and the partition functions are canceled. However, this derivation of DPO is based on the assumption that optimal solution is achieved, which in practice is not guaranteed especially in the early stage training\citep{ji2024towards}. 
This issue leads to a compromised approximation of the optimal policy, especially when the optimality of $\pi_\theta$ is not achieved in the early stage training.

\begin{table*}[!tb]
\centering
\small
\caption{Summary of some typical $f$-divergences $\mathbb{D}_f(p\|q)$ together with generator
functions. 
Part of the list of divergences are borrowed from~\citet{nielsen2013chi,nowozin2016f}.
For all divergences we have the generator function $f: \mathbb{R}^+ \to \mathbb{R} \cup
\{+\infty\}$ strictly convex and lower-semicontinuous. The function also satisfies
$f(1)=0$ which ensures that $\mathbb{D}_f(p\|q)=0$ when $p(x)=q(x)$.
}
\resizebox{\linewidth}{!}{%
\begin{tabular}{lll}
\toprule
Name & $\mathbb{D}_f(p\|q)$ & Generator $f(u)$ 
\\ 
\midrule

Kullback-Leibler
& $\int p(x) \log \frac{p(x)}{q(x)} \,\textrm{d}x$
& $u \log u$
\\
Reverse Kullback-Leibler
& $\int q(x) \log \frac{q(x)}{p(x)}\,\textrm{d}x$
& $-\log u$
\\

$\alpha$-divergence ($\alpha \notin \{0,1\}$)
& $\frac{1}{\alpha (\alpha-1)} \int
  \left(q(x) \left[\left(\frac{p(x)}{q(x)}\right)^{1-\alpha}- (1 - \alpha)\left(\frac{p(x)}{q(x)}\right) -\alpha\right]\right)
  \,\textrm{d}x$
& $\frac{1}{\alpha(\alpha-1)} \left(u^{1-\alpha} - (1-\alpha) u - \alpha\right)$
\\

Jeffrey   
& $\int\left(p(x)-q(x)\right) \log \left(\frac{p(x)}{q(x)}\right) \,\textrm{d}x$
& $(u-1) \log u$
\\
Jensen-Shannon
& $\frac{1}{2} \int p(x) \log \frac{2 p(x)}{p(x)+q(x)}
  + q(x) \log \frac{2 q(x)}{p(x) + q(x)}\,\textrm{d}x$
& $-(u+1) \log \frac{1+u}{2} + u \log u$
\\


Squared Hellinger
& $\int\left(\sqrt{p(x)} - \sqrt{q(x)}\right)^2 \,\textrm{d}x$
& $\left(\sqrt{u}-1\right)^2$
\\


\bottomrule
\end{tabular}
}%
\label{tab:f-divergences}
\end{table*}

\subsection{$f$-divergence}
\label{subsec:prelim-fdiv}

\footnotetext{Note that,
to facilitate using $f$-divergence framework, 
in this paper we describe divergences by $\FDiv(\pi_\theta\| \pi^*)$ instead of the conventional order $\FDiv(\pi^* \|\pi_\theta)$,
which results in interpretations of the divergence behaviors different from other literature~\citep{ji2024towards}.
}

Divergence describes the difference between two probability distributions.
A general family of divergences is the $f$-divergences~\citep{csiszar2004information,liese2006divergences}, also
known as the Ali-Silvey distances~\citep{ali1966general}.
For a convex function $f: \R^+\rightarrow \R$ that is lower-semicontinuous, strictly convex around $1$, and satisfies $f(1) = 0$, given two distributions $p$ and $q$, the corresponding $f$-divergence for these two distributions is defined as:
\begin{equation}
    \FDiv(p \| q) = \expect_{q(x)}\left[f\left(\frac{p(x)}{q(x)}\right)\right],
\end{equation}
where $f$ is called \emph{generator function}.
Different choices of $f$-divergence can cover a wide class of popular divergences, including forward and reverse Kullback-Leibler (KL) divergence, Jensen-Shannon~(JS) divergence, and Jeffrey's divergence, etc. We provide a summary of their analytic forms and corresponding generator functions in Table~\ref{tab:f-divergences}. An illustration of different behaviors of these divergences is presented in Fig.~\ref{fig:six-figures}.

\section{Methodology}
\label{method:general_section}

In this section, we present our proposed \methodabbrv in details. We start by introducing the \methodabbrv framework in \Cref{subsec:method-fpo}, which conducts alignment under the probability matching perspective with $f$-divergence. In \Cref{subsec:method-general}, we show practical \methodabbrv implementation given common pair-wise preference data without reward labeling, and in \Cref{subsec:method-instances}, we further present several instances of \methodabbrv with specified $f$-divergence, covering DPO as a specific case. In \Cref{subsec:extention}, we further discuss the connections between \methodabbrv and several recent heuristic alignment methods, and further introduce several variants of \methodabbrv.

\subsection{The~\methodabbrv Framework}
\label{subsec:method-fpo}

As shown in \Cref{sec:prelim}, the goal of preference optimization is to fine-tune the LLM $\pi_\theta$ toward the optimal policy $\pi^*$, \textit{i.e.}, minimizing the distances between $\pi_\theta$ and $\pi^*$. In this paper, we explicitly define the LLM alignment task as a distribution matching problem:
\begin{theorem}
    Let $\hat{\pi}_\theta(y|x) \propto \pi_\theta(y|x)^{\beta}  \pref(y|x)^{(1-{\beta})}$ and $\hat{\pi}^* \propto \pref \exp{\left(r(x,y)\right)}$. We define our alignment objective as minimizing the following $f$-divergence:
    \begin{equation}
        \FDiv(\hat{\pi}_\theta\|\hat{\pi}^*) 
        =\expect_{\hat{\pi}^*} \left[f\left(\frac{\hat{\pi}_\theta}{\hat{\pi}^*}\right)\right]
        =\expect_{\pref} \left[\frac{\hat{\pi}^*}{\pref}f\left(\frac{\hat{\pi}_\theta}{\hat{\pi}^*}\right)\right].
    \end{equation}
    With unlimited model capacity and perfect optimization, the optimal policy $\pi^*_\theta$ satisfies that $\pi^*_\theta = \pi^*$.
\label{thm:optimal-pi}
\end{theorem}
We leave the full proof in~\Cref{app:sec:proof}. Intuitively, this objective aims to optimize $\hat{\pi}_\theta(y|x)$ towards $\hat{\pi}^*$, which is equivalent to optimizing ${\pi}_\theta(y|x)$ towards ${\pi}^*$ in \Cref{eq:pi-optimal}. 
According to the definition of $\hat{\pi}^*$, we have that the reward function can be written as $r(x, y)=\log\hat{\pi}^*(y|x) - \log{\pref(y|x)}$.
By further defining the log ratio $g_\theta(x, y)= \log \hat{\pi}_\theta(y|x) - \log{\pref(y|x)} = \beta (\log {\pi}_\theta(y|x) - \log{\pref(y|x)})$, the above objective can be simplified as:
\begin{equation}
\label{eq:fpo-simplified}
    \FDiv(\hat{\pi}_\theta\|\hat{\pi}^*) =\expect_{\pref} \left[\pi^r f\left(\frac{\pi^{g_\theta}}{\pi^r}\right)\right],
\end{equation}
where we define $\pi^{g_\theta} = \frac{1}{Z_{g_\theta}(x)} \exp(g_\theta(x,y))$ and $\pi^r = \frac{1}{Z_{r}(x)} \exp(r(x,y))$, with $Z_{g_\theta}$ and $Z_{r}$ being the partition functions. In practice, it is intractable to compute these partition functions in the high dimensional space. Instead, we take Monte Carlo estimation by the multiple samples from the preference dataset. In general case, for every prompt $x$ we have $K$ \textit{i.i.d.} completions $y_{1:K}=\{y_{1},\cdots, y_{K}\}$ drawn from $\pref(y|x)$, where we approximate the distributions $\pi^{g_\theta}$ and $\pi^r$ by normalizing the exponential rewards over the $K$ samples:
\begin{align}
    & \pi^{g_\theta}({i}|y_{1:K}, x) = \frac{e^{g_\theta(x, y_{i})}}{\sum_{j=1}^K e^{g_\theta(x, y_{j})}}, \\
    & \pi^{r}({i}|y_{1:K}, x) = \frac{e^{r(x, y_{i})}}{\sum_{j=1}^K e^{r(x, y_{j})}}.
\end{align}
Plugging the above expression into \Cref{eq:fpo-simplified}, we have the complete form of \methodabbrv objective $\gL_\textnormal{\methodabbrv}$ as follows:
\begin{equation}
\begin{aligned}
\label{eq:loss_fpo_general}
    & \gL_\textnormal{\methodabbrv} (\pi_\theta) =\expect_{x \sim \gD} \expect_{y_{1:K} \sim \pref} \Bigg[ \\
    & \quad \frac{e^{r(x, y_{i})}}{\sum_{j=1}^K e^{r(x, y_{j})}} f\left(\frac{e^{g_\theta(x, y_{i})}}{\sum_{j=1}^K e^{g_\theta(x, y_{j})}} \Big/ \frac{e^{r(x, y_{i})}}{\sum_{j=1}^K e^{r(x, y_{j})}} \right)\Bigg].
\end{aligned}
\end{equation}
Formally, the approximated objective with $K$ samples enjoys the following theoretical property, 
\begin{theorem}
\label{thm:divergence}
    With $K \rightarrow \infty$,  asymptotically we have
    \begin{equation}
        \gL_\textnormal{\methodabbrv}(\pi_\theta) = \expect_{x \sim \gD} \FDiv \left(\hat{\pi}_\theta(y|x)\|\hat{\pi}^*(y|x)\right),
    \end{equation}
    where $\hat{\pi}_\theta(y|x)$ and $\hat{\pi}^*(y|x)$ follows the definition in \Cref{thm:optimal-pi}.
\end{theorem}
The theorem states that the general \methodabbrv objective is an unbiased optimizer for $f$-divergences between $\hat{\pi}_\theta$ and $\hat{\pi}^*$. In the following, we give the practical form of $\gL_\textnormal{\methodabbrv}$ under typical preference optimization scenario.

\subsection{Generalized Preference Optimization}
\label{subsec:method-general}

Under the distribution matching perspective, we have presented the general \methodabbrv framework where we use $f$-divergence to align $\pi_\theta$, with $K$ data completions and an underlying reward model $r$. In this section, we show the instance of \methodabbrv objective under practical preference optimization settings:

\textbf{Pair-wise Preference Data}. The most common preference dataset typically consists of pair-wise completions for each prompt, \textit{i.e.}, setting $K=2$. Since there are two completions $y$ for each prompt $x$, we can denote them as ``winning" sample $y_w$ and ``losing" sample $y_l$. Then~\Cref{eq:loss_fpo_general} can be simplified as:
\begin{equation}
{
\small
\begin{aligned}
\label{eq:f-div_pair-data}
    \gL_\textnormal{\methodabbrv}& (\pi_\theta) = \expect_{x \sim \gD} \expect_{y_{1:2} \sim \pref}  \left[\pi^r f\left( { \frac{
        e^{g_\theta(x, y_{i})}
    }{
        \sum_{j=1}^2
        e^{g_\theta(x, y_{j})}
    }  } \Big/ {\pi^r}\right)\right], \\
    = & \expect_{x \sim \gD} \expect_{\{y_w, y_l\} \sim \pref} \Bigg[ \pi^{r_w} f \left(\frac{ \sigmoid(g_\theta(x, y_w)-g_\theta(x, y_l)) }{\pi^{r_w}}\right) \\
    & \quad \quad \quad \quad \quad \quad + \pi^{r_l} f \left(\frac{ \sigmoid(g_\theta(x, y_l)-g_\theta(x, y_w)) }{\pi^{r_l}}\right)\Bigg],
\end{aligned}
}
\end{equation}
where $\pi^{r_w}=\sigma(r_w-r_l)$ and $\pi^{r_l}=\sigma(r_l-r_w)$ are self-normalized reward values for winning and losing samples, respectively. Such formulation follows the Bradley-Terry model~\citep{bradley1952rank}. 

\textbf{Preference Data without Reward Value}. In common cases, the preference data is accessible without reward label. We follow the BT model form and turn $\pi^{r_w}$ and $\pi^{r_l}$ to binary supervision labels, \textit{i.e.}, we set $\pi^{r_w} \approx 1$ and $\pi^{r_l} \approx 0$. To avoid the numerical issue, we smooth the labels by setting $\pi^{r_w} = 1 - \epsilon$ and $\pi^{r_l} = \epsilon$, where $\epsilon >0$ is a hyperparameter controlling the smoothness. Then we further simplify~\Cref{eq:f-div_pair-data} as:
\begin{equation}
{\small
\begin{aligned}
\label{eq:fpo}
    \gL_\textnormal{\methodabbrv}(\pi_\theta) & =\expect_{x \sim \gD} \expect_{\{y_w, y_l\} \sim \pref} \Bigg[ \\
    & (1-\epsilon) f \left(\frac{ \sigmoid(g_\theta(x, y_w)-g_\theta(x, y_l)) }{1-\epsilon}\right) \\
    & + \epsilon f \left(\frac{ \sigmoid(g_\theta(x, y_l)-g_\theta(x, y_w)) }{\epsilon}\right)\Bigg].
\end{aligned}
}
\end{equation}
By far we have derived the \methodabbrv objective with common pair-wise preference data. By plugging in any function $f$ that satisfies the requirement in \Cref{subsec:prelim-fdiv}, \textit{e.g.} the generator functions in \Cref{tab:f-divergences}, we can get practical \methodabbrv optimization objectives. We will introduce several instances of \methodabbrv in the following section.

\subsection{Instances of \methodabbrv}
\label{subsec:method-instances}

In this section, we show several instances of \methodabbrv with specified $f$-divergences. We first briefly show that when taking the reverse KL divergence as instance, \methodabbrv can recover the original DPO objective. Then we present the specific instance with $\alpha$-divergence, \textit{a.k.a} $\alpha$-PO, as an example, which achieves the most competitive results in our experiments.

\textbf{DPO as \methodabbrv with Reverse KL}. First, we show that from our $f$-divergence minimization perspective, it can be easily demonstrated that DPO~\citep{rafailov2023direct} corresponds to minimizing the reverse KL divergence $\RKLDiv(\pi_\theta\|\pi^*)$, as briefly derived here:
\begin{proof}
By realizing the function $f$ in \Cref{eq:fpo} as the generator function of reverse KL, \textit{i.e.}, $-\log u$, we have:
\begin{align*}
\small
    & \gL_\textnormal{\methodabbrv-RKL}(\pi_\theta) =\expect_{x \sim \gD} \expect_{\{y_w, y_l\} \sim \pref} \big[ \\
    & \quad \quad - (1- \epsilon) \log \left({ \sigmoid(g_\theta(x, y_w)-g_\theta(x, y_l)) }\right) \\
    & \quad \quad - \epsilon \log \left({ \sigmoid(g_\theta(x, y_l) - g_\theta(x, y_w)) }\right) \\
    & \quad \quad - (1- \epsilon) \log(1-\epsilon) - \epsilon \log \epsilon \big].
\end{align*}
When $\epsilon \rightarrow 0$, we have that $\lim_{\epsilon \rightarrow 0} (1- \epsilon) \log(1-\epsilon) = 0$ and $\lim_{\epsilon \rightarrow 0} \epsilon \log \epsilon = 0$. Besides, we also have $\lim_{\epsilon \rightarrow 0} \epsilon \log \left({ \sigmoid(g_\theta(x, y_l) - g_\theta(x, y_w)) }\right) = 0$. Therefore, without reward label smoothing, the objective is:
\begin{align*}
    \lim_{\epsilon \rightarrow 0} \gL_\textnormal{\methodabbrv-RKL}(\pi_\theta) = - \log \left({ \sigmoid(g_\theta(x, y_w)-g_\theta(x, y_l)) }\right),
\end{align*}
which recovers the original DPO objective \Cref{eq:dpo-obj}.
\end{proof}
Similarly, when combined with forward KL function $u \log u$, the \methodabbrv instance corresponds to EXO~\citep{ji2024towards}. The derivation is provided in~\Cref{app:sec:proof} due to limited space, akin to the DPO derivation above but with a different generator function.

\textbf{\methodabbrv with $\alpha$-divergence}. We present the alignment objective when using $\alpha$-divergence, \textit{a.k.a} $\alpha$-PO, as an example. By plugging the generator function of $\alpha$-divergence into \Cref{eq:fpo}, we have the objective as:
\begin{equation}
\begin{aligned}
    \gL_\textnormal{\methodabbrv}&(\pi_\theta) =\expect_{x \sim \gD} \expect_{\{y_w, y_l\} \sim \pref} \Big[ \frac{1}{\alpha(\alpha-1)} \\
    & (1-\epsilon) \left(u_1^{1-\alpha} - (1-\alpha) u_1 - \alpha\right) \\
    & + \epsilon \left(u_2^{1-\alpha} - (1-\alpha) u_2 - \alpha\right) \Big],
\end{aligned}
\end{equation}
where the densities ratios $u$ are defined as $u_1 = \frac{ \sigmoid(g_\theta(x, y_w)-g_\theta(x, y_l)) }{1-\epsilon}$ and $u_2 = \frac{ \sigmoid(g_\theta(x, y_l)-g_\theta(x, y_w)) }{\epsilon}$, respectively. $\alpha$-divergence covers a family of divergences with varying $\alpha$ values. As $|\alpha|$ increases, the divergence becomes more sensitive to differences in the tails of the distributions. Specifically, it converges to KL divergence when $\alpha \rightarrow 0$ and reverse KL divergence when $\alpha \rightarrow 1$. This property allows us to interpolate between KL and reverse KL. In our empirical study, we achieve the most competitive results using $\alpha$ divergence among all \methodabbrv variants.

\begin{figure}[t!]
    \centering
    \includegraphics[width=0.96\linewidth]{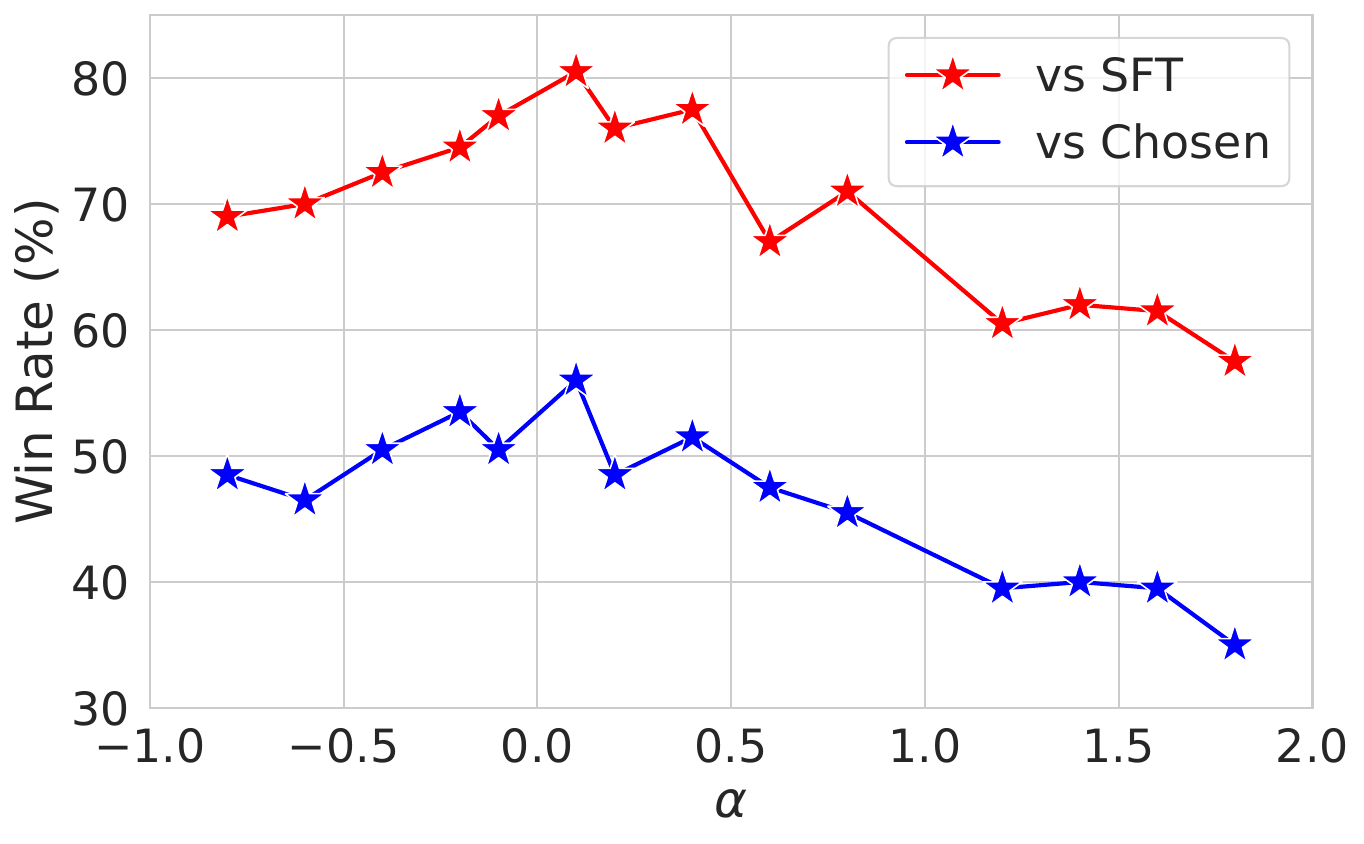}
    \caption{The win rate (\%) against SFT and Chosen of Pythia-2.8B fine-tuned by~\methodabbrv on Anthropic HH dataset with different values of $\alpha$.}
    \label{fig:alpha-sweep}
\end{figure}

\subsection{Empirical Variants with Approximations}
\label{subsec:extention}

As shown in the objective $\gL_\textnormal{\methodabbrv} (\pi_\theta)$ in~\Cref{eq:fpo}, we conduct optimization over the log density odds $g_\theta(x, y_w)-g_\theta(x, y_l) = \beta (\log {\pi}_\theta(y_w|x) - \log{\pref(y_w|x)} - \log {\pi}_\theta(y_l|x) + \log{\pref(y_l|x)})$. This expression is reported to have several drawbacks~\citep{meng2024simpo}: 1) the presence of $\pref$ incurs additional training cost, and 2) the objective mismatch the generation metric where sentence likelihood is averaged by sentence length. To this end, heuristic approximations have been proposed, which yield better empirical results:
\begin{equation}
\small
\label{eq:simpo-style}
    \hat{g}_\theta(x, y_w)-\hat{g}_\theta(x, y_l) = \frac{\beta}{|y_w|} \log {\pi}_\theta(y_w|x) - \frac{\beta}{|y_l|} \log {\pi}_\theta(y_l|x) - \gamma,
\end{equation}
where $|y|$ is the sentence length of corresponding completions $y$, and $\gamma$ is a target margin hyperparameter approximating $\beta(\log{\pref(y_w|x)} - \log{\pref(y_l|x)})$. Recent progress shows that this practical implementation simplifies the training objective and greatly improves alignment performance~\citep{meng2024simpo}. Thanks to the generality of our framework, empirically we can take this approximation into our method by substituting the $g_\theta$ odds in \Cref{eq:fpo} with the $\hat{g_\theta}$ one above. This empirical variant enjoys improved training efficiency and also leads to significantly higher performance.

\section{Experiments}

In this section, we start by investigating the choices of the $f$-divergence in our~\methodabbrv in~\Cref{exp:function_f_ablation} that helps us identity the most performant form, which is further evaluated on popular LLM benchmarks in~\Cref{exp:simpo_table}.



\subsection{Experiments on the Choice of $f$ in~\methodabbrv}
\label{exp:function_f_ablation}

\begin{figure}[t!]
    \centering
    \includegraphics[width=0.96\linewidth]{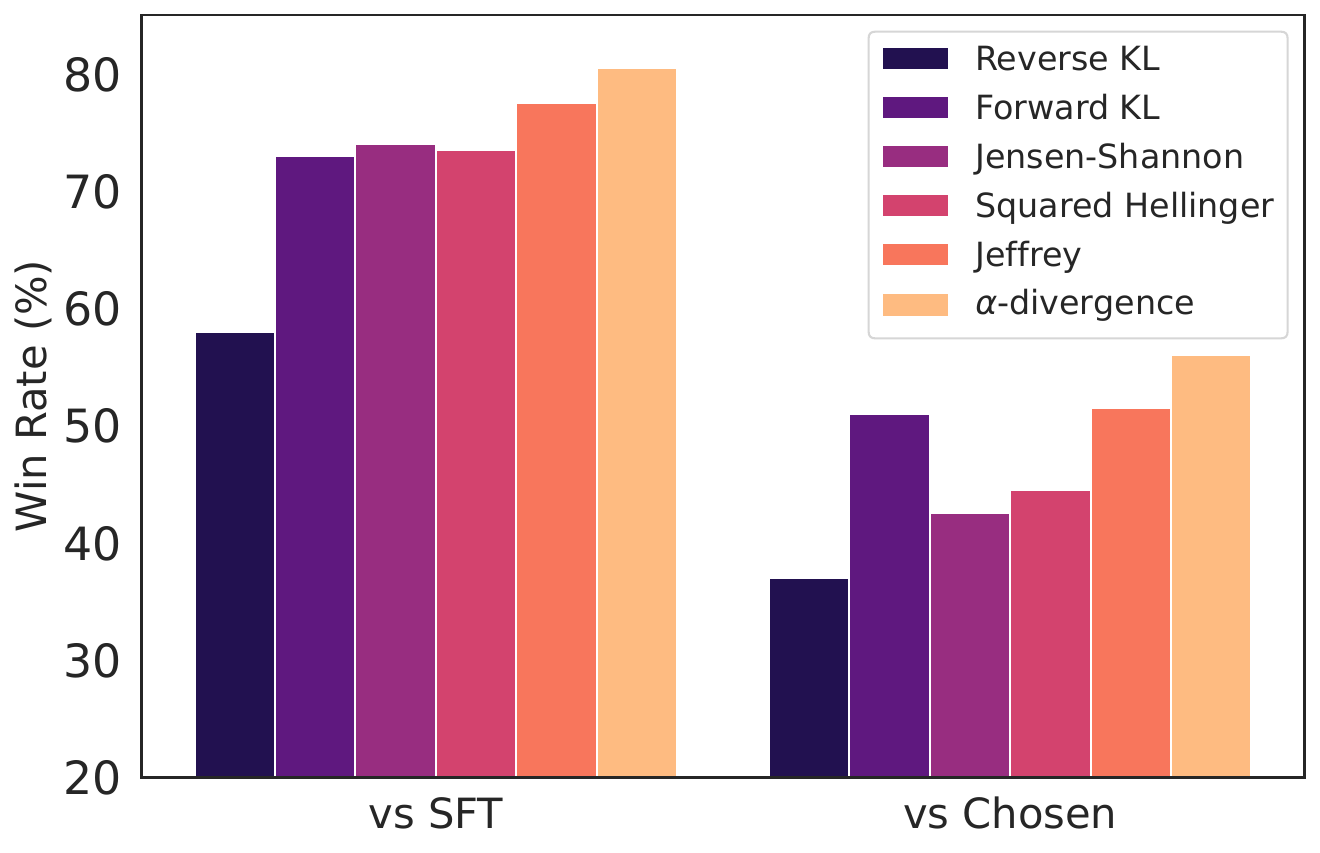}
    \caption{The win rate (\%) of Pythia-2.8B fine-tuned by 6 different $f$-divergences using \methodabbrv on Anthropic HH dataset. $\alpha$-divergence generally performs the best.}
    \label{fig:different-divs}
\end{figure}

In this subsection, we empirically investigate the effect of different instantiations of $f$ in performing preference optimization of LLMs. We detail our setups as follows.

\textbf{Datasets.} Here we utilize two datasets, the Reddit TL;DR summarization dataset ~\citep{volske-etal-2017-tl} and the Anthropic Helpful and Harmless (HH) dialogue datasets ~\citep{bai2022traininghelpfulharmlessassistant}, in this subsection.  The TL;DR dataset contains Reddit posts with user-generated summaries. We employ the filtered dataset \citep{stiennon2020learning} to train the SFT model, and the preference dataset for preference optimization. The HH dataset provides multi-turn dialogues between users and AI assistants. We leverage the chosen response for SFT model and the helpful subset of this preference dataset for preference optimization.

\textbf{Setups.} Following~\citet{ji2024towards}, we explore two alignment protocols for each dataset. \textbf{1.} Direct preference training (\texttt{w/ Preferences}): We employ a dataset $\gD_{\text{pref}}$ with data points in the form of $(x, y_w, y_l)$, where $x$ is the input, $y_w$ and $y_l$ represent the preferred and rejected response, respectively.
\textbf{2.} Reward-based training (\texttt{w/ Reward Model}): We employ a dataset $\gD_{\text{rw}}$, where each instance comprises an input $x$ and $K$ pairs of $(y_k, r_k)$, with $y_k$ being a response generated by the SFT model, and $r_k$ being the corresponding reward labeled by certain reward model trained on the preference dataset $\gD_{\text{pref}}$.

\textbf{Model and evaluation.} Here we use Pythia-2.8B~\citep{biderman2023pythia} as the backbone to perform preference optimization. For evaluation, we adopt GPT-4 for zero-shot pair-wise comparisons of outputs generated by our current model against those produced by (1) the SFT model and (2) the preferred responses from our original preference dataset. We reuse the evaluation prompt from \cite{ji2024towards} detailed in \Cref{app:sec:prompt_details}, which has been shown to align closely with human judgments~\citep{rafailov2023direct}.

\begin{table}[!t]
  \centering
  \caption{The win rate (\%) against SFT and Chosen of Pythia-2.8B fine-tuned by various alignment algorithms on TLDR and Anthropic HH datasets. Results of baselines are taken from~\citet{ji2024towards}. Best numbers of preference optimization methods are boldfaced.}
  \resizebox{\columnwidth}{!}{
    \begin{tabular}{lcccc}
    \toprule
          & \multicolumn{2}{c}{\textbf{TLDR}} & \multicolumn{2}{c}{\textbf{Anthropic HH}} \\
          & \textbf{vs SFT} & \textbf{vs Chosen}  & \textbf{vs SFT}  & \textbf{vs Chosen}  \\
    \midrule
           \multicolumn{5}{c}{\small{\texttt{w/ Preferences}}} \\
    \midrule
    $\text{DPO}_\text{pref}$ & 57.0      &   30.5    & 58.0      & 37.0 \\
    $\text{EXO}_\text{pref}$ &   83.0    &    55.0   &  73.0     & 51.0 \\
    $f\text{-PO}_\text{pref}$ &   \textbf{84.0}    &    \textbf{61.5}  &   \textbf{80.5}    & \textbf{56.0} \\
    \midrule
           \multicolumn{5}{c}{\small{\texttt{w/ Reward Model}}} \\
    \midrule
    Best-of-$N$ &   83.5    &  60.0     &   86.0    &  63.0 \\
    PPO   &   77.0    &    52.0   &   66.5    &  52.0 \\
    $\text{DPO}_\text{rw}$ &   70.0    &   41.0    &   75.5    &  49.0 \\
    $\text{EXO}_\text{rw}$ &   84.5    &   \textbf{64.0}    &   83.5    &  60.0 \\
    $f\text{-PO}_\text{rw}$ &  \textbf{87.5}   &   60.0    &    \textbf{85.0}   & \textbf{63.5} \\
    \bottomrule
    \end{tabular}%
    }
  \label{tab:f_function_ablation}%
  \vskip -0.1in
\end{table}%

\textbf{The effect of $\alpha$ in $\alpha$-PO.} We first examine the impact of $\alpha$ in $\alpha$-PO in the preference training setup by performing a sweep over $\alpha$, with results illustrated in Fig.~\ref{fig:alpha-sweep}. We find that by controlling $\alpha$ we are able to effectively maneuver between mode-seeking and mode-covering, which results in a smooth change of win rate. Notably, the performance approximates that of forward KL when $\alpha$ approaches 0 and that of reverse KL when $\alpha$ becomes close to 1 (\emph{c.f.} Fig.~\ref{fig:different-divs}), aligning with our theoretical analysis. We obtain the best result with $\alpha=0.1$ under this experimental setup.

\textbf{The effect of different $f$ in the $f$-divergence family.} We further study the empirical performance when optimizing the model with different $f$ enumerated in Table~\ref{tab:f-divergences}. For $\alpha$-divergence, we use $\alpha$ that leads to the best performance as investigated in the previous study. As depicted in Fig.~\ref{fig:different-divs}, we observe that $\alpha$-divergence generally contributes to the highest win rate with up to $>20\%$ enhancement compared with reverse KL (\emph{i.e.}, DPO), while some other types of $f$-divergences, \emph{e.g.}, Jeffrey's divergence, also demonstrate competitive performance. Inspired by these results, we will prioritize $\alpha$-PO in the benchmarks in the following sections due to its superior performance.

\textbf{Overall comparison.} We thoroughly compare~\methodabbrv against several baselines (\emph{e.g.}, DPO and EXO) in both \texttt{w/ Preferences} and \texttt{w/ Reward Model} setting with results displayed in Table~\ref{tab:f_function_ablation}. For the reward model setting, we additionally report the numbers of PPO~\citep{schulman2017proximal} and Best-of-$N$, which samples $N = 128$ outputs from the SFT
policy and then selects the highest scored response judged by the reward model.
As shown in~\Cref{tab:f_function_ablation}, our $f$-PO outperforms both DPO and EXO by a significant margin in both settings across the two datasets with, for instance, an average of 4\% gain in win rate against SFT model and 6\% against the chosen response over $\text{EXO}_\text{pref}$. Notably, these improvements generalize to reward-based training paradigms, underscoring the method's robustness. The substantial enhancements validate our theoretical framework and highlight the practical advantages of our generalized approach towards preference optimization of LLMs.

\subsection{Benchmarking~\methodabbrv}
\label{exp:simpo_table}

In this subsection, we further demonstrate the strong performance of our~\methodabbrv on four popular benchmarks through four different settings.

\begin{table*}[!t]
\setlength{\tabcolsep}{5pt}
\centering
\small
\caption{Results on AlpacaEval 2~\citep{AlpacaEval}, Arena-Hard~\citep{li2024crowdsourced}, and MT-Bench~\citep{zheng2023judging} under four settings. LC and WR denote length-controlled and raw win rate, respectively. SFT models are trained on UltraChat dataset for Base settings and initialized as off-the-shelf models for Instruct settings. Results of baselines are taken from~\citet{meng2024simpo}.}
\vskip -0.1in
\resizebox{\textwidth}{!}{
\begin{tabular}{lcccccccccc}
\toprule
\multirow{3}{*}{\textbf{Method}} & \multicolumn{5}{c}{\textbf{Mistral-Base (7B)}} & \multicolumn{5}{c}{\textbf{Mistral-Instruct (7B)}} \\ 
\cmidrule(lr){2-6}\cmidrule(lr){7-11}
& \multicolumn{2}{c}{\textbf{AlpacaEval 2}} & \multicolumn{1}{c}{\textbf{Arena-Hard}} & \multicolumn{2}{c}{\textbf{MT-Bench}} & \multicolumn{2}{c}{\textbf{AlpacaEval 2}} & \multicolumn{1}{c}{\textbf{Arena-Hard}} & \multicolumn{2}{c}{\textbf{MT-Bench}} \\
\cmidrule(lr){2-3}\cmidrule(lr){4-4} \cmidrule(lr){5-6} \cmidrule(lr){7-8}\cmidrule(lr){9-9}\cmidrule(lr){10-11} 
& {\scriptsize \bf LC (\%)} & {\scriptsize \bf WR (\%)} & {\scriptsize \bf WR (\%)} & {\scriptsize \bf GPT-4 Turbo} & {\scriptsize \bf GPT-4} & {\scriptsize \bf LC (\%)}  & {\scriptsize \bf WR (\%)} & {\scriptsize \bf WR (\%)} & {\scriptsize \bf GPT-4 Turbo} & {\scriptsize \bf GPT-4} \\
\midrule
SFT &  8.4 & 6.2 & 1.3 & 4.8 & 6.3 & 17.1 & 14.7 & 12.6 & 6.2 & 7.5 \\
\midrule
RRHF   & 11.6 & 10.2 &  5.8 & 5.4 & 6.7 & 25.3 & 24.8 & 18.1 & 6.5 & 7.6 \\
SLiC-HF & 10.9 &  8.9 &  7.3 & 5.8 & \textbf{7.4} & 24.1 & 24.6 & 18.9 & 6.5 & \bf 7.8 \\
DPO & 15.1 & 12.5 & 10.4 & 5.9 & 7.3 & 26.8 & 24.9 & 16.3 & 6.3 & 7.6 \\
IPO & 11.8 & 9.4 & 7.5 & 5.5 & 7.2 & 20.3 & 20.3 & 16.2 & 6.4 & \bf 7.8 \\
CPO &  9.8 &  8.9 &  6.9 & 5.4 & 6.8 & 23.8 & 28.8 & \textbf{22.6} & 6.3 & 7.5 \\
KTO & 13.1 & 9.1 & 5.6 & 5.4 & 7.0 & 24.5 & 23.6 & 17.9 & 6.4 & 7.7 \\
ORPO & 14.7 & 12.2 & 7.0 & 5.8 & 7.3 & 24.5 & 24.9 & 20.8 & 6.4 & 7.7 \\
R-DPO  & 17.4 & 12.8 & 8.0 & 5.9 & \textbf{7.4} & 27.3 & 24.5 & 16.1 & 6.2 & 7.5 \\
SimPO & 21.5 & 20.8 & 16.6 & \bf 6.0 & 7.3 & 32.1 & 34.8 & 21.0 & \bf 6.6 & 7.6 \\
\midrule
$f$-PO & \textbf{23.7} & \textbf{22.0} & \textbf{16.8} & \textbf{6.0} & 
\textbf{7.4} & \textbf{32.9} & \textbf{35.8} & 21.5 & \bf 6.6 & 7.6 \\
\midrule[.7pt]
\multirow{3}{*}{\textbf{Method}} & \multicolumn{5}{c}{\textbf{Llama3-Base (8B)}} & \multicolumn{5}{c}{\textbf{Llama3-Instruct (8B)}} \\ 
\cmidrule(lr){2-6}\cmidrule(lr){7-11}
& \multicolumn{2}{c}{\textbf{AlpacaEval 2}} & \multicolumn{1}{c}{\textbf{Arena-Hard}} & \multicolumn{2}{c}{\textbf{MT-Bench}} & \multicolumn{2}{c}{\textbf{AlpacaEval 2}} & \multicolumn{1}{c}{\textbf{Arena-Hard}} & \multicolumn{2}{c}{\textbf{MT-Bench}} \\
\cmidrule(lr){2-3}\cmidrule(lr){4-4} \cmidrule(lr){5-6} \cmidrule(lr){7-8}\cmidrule(lr){9-9}\cmidrule(lr){10-11} 
& {\scriptsize \bf LC (\%)} & {\scriptsize \bf WR (\%)} & {\scriptsize \bf WR (\%)} & {\scriptsize \bf GPT-4 Turbo} & {\scriptsize \bf GPT-4} & {\scriptsize \bf LC (\%)}  & {\scriptsize \bf WR (\%)} & {\scriptsize \bf WR (\%)} & {\scriptsize \bf GPT-4 Turbo} & {\scriptsize \bf GPT-4} \\
\midrule
SFT & 6.2 & 4.6 & 3.3 & 5.2 & 6.6 & 26.0 & 25.3 & 22.3 & 6.9 & 8.1 \\
\midrule
RRHF & 12.1 & 10.1 &  6.3 & 5.8 & 7.0 & 31.3 & 28.4 & 26.5 & 6.7 & 7.9 \\
SLiC-HF & 12.3 & 13.7 &  6.0 & 6.3 & 7.6 & 26.9 & 27.5 & 26.2 & 6.8 & 8.1 \\
DPO & 18.2 & 15.5 & 15.9 & 6.5 & 7.7 & 40.3 & 37.9 & 32.6 & 7.0 & 8.0 \\
IPO & 14.4 & 14.2 & 17.8 & 6.5 & 7.4 & 35.6 & 35.6 & 30.5 & 7.0 & \textbf{8.3} \\
CPO  & 10.8 &  8.1 &  5.8 & 6.0 & 7.4 & 28.9 & 32.2 &28.8  & 7.0 & 8.0 \\
KTO & 14.2 & 12.4 & 12.5 & 6.3 & \textbf{7.8} & 33.1 & 31.8 & 26.4 & 6.9 & 8.2 \\
ORPO & 12.2 & 10.6 & 10.8 & 6.1 & 7.6 & 28.5 & 27.4 & 25.8 & 6.8 & 8.0 \\
R-DPO & 17.6 & 14.4 & 17.2 & \textbf{6.6} & 7.5 & 41.1 & 37.8 & 33.1 & 7.0 & 8.0 \\
SimPO & 22.0 & 20.3 & 23.4 & \textbf{6.6} & 7.7 & 44.7 & 40.5 & \textbf{33.8} & 7.0 & 8.0 \\ 
\midrule
$f$-PO & \textbf{23.5} & \textbf{20.9} & \textbf{28.2} & \textbf{6.6} & \textbf{7.8} & \textbf{45.3} & \textbf{41.0} & 33.5 & \textbf{7.1} & 8.2 \\ 
\bottomrule
\end{tabular}
}
\label{tab:main_table}
\vspace{-1pt}
\end{table*}

\begin{table*}[t!]
  \centering
  \caption{Results on LLM Leaderboard v2 with Llama3-Instruct (8B) and Mistral-Base (7B).}
  \vspace{-8pt}
  \small
    \begin{tabular}{lccccccc}
    \toprule
          & IFEval & BBH   & MATH Lvl 5 & GPQA  & MUSR  & MMLU-PRO & Average \\
    \midrule
    SimPO-Llama3-8B-Instruct & 65.04 & 26.71 & 2.57  & \textbf{5.82} & 8.15  & 27.66 & 22.66 \\
    \methodabbrv-Llama3-8B-Instruct & \textbf{67.69} & \textbf{28.75} & \textbf{4.98} & 4.92  & \textbf{9.08} & \textbf{28.79} & \textbf{24.04} \\
    \midrule
    SimPO-Mistral-7B-Base & 47.01 & 22.33 & 0.60   & 4.47  & 8.03  & 18.91 & 16.89 \\
    \methodabbrv-Mistral-7B-Base & \textbf{48.46} & \textbf{24.16} & \textbf{1.66} & \textbf{6.38} & \textbf{12.73} & \textbf{19.72} & \textbf{18.85} \\
    \bottomrule
    \end{tabular}%
  \label{tab:openllm}%
\end{table*}%

\textbf{Models and training details.} We adopt two base models: Llama-3-8B \citep{dubey2024llama3herdmodels} and Mistral-7B \citep{jiang2023mistral7b}, under two setups, \emph{Base} and \emph{Instruct}, following the approach outlined in \cite{meng2024simpo}. 
In the \emph{Base} setup, we first train an SFT model on UltraChat-200k \citep{ding2023enhancing} before performing preference optimization on UltraFeedback \citep{cui2023ultrafeedback}, offering high reproducibility with open-source data and methods.
The \emph{Instruct} setup instead utilizes publicly available instruction-tuned models (Llama-3-8B-Instruct and Mistral-7B-Instruct) as the SFT models, which are more performant but less transparent due to undisclosed finetuning procedures. We use specifically constructed preference datasets
for Llama-3 and Mistral, following~\cite{meng2024simpo}. More details are deferred to \Cref{app:sec:experiment_details}.

This leads to four configurations in total: Llama-3-Base, Llama-3-Instruct, Mistral-Base, and Mistral-Instruct. For our~\methodabbrv, we employed the $\alpha$-divergence algorithm family with a properly tuned $\alpha$ parameter, inspired by the study in~\Cref{exp:function_f_ablation}. Notably, we also utilize the SimPO-style parameterization of~\methodabbrv in~\Cref{eq:simpo-style} due to its superior performance, and report the best performance over a range of $\alpha$. Detailed hyperparameters can be found in \Cref{app:sec:experiment_details}.

\textbf{Evaluation metrics.} We evaluate our models using four popular open-ended instruction-following benchmarks: AlpacaEval 2~\citep{dubois2024length}, MT-Bench~\citep{zheng2023judging}, Arena-Hard 0.1~\citep{li2024crowdsourced}, and Open LLM Leaderboard v2~\citep{open-llm-leaderboard-v2}. These benchmarks assess diverse conversational abilities across various queries. For AlpacaEval 2, we report both the raw win rates (WR) and length-controlled win rates (LC). MT-Bench scores are averaged using GPT-4 and GPT-4-Preview-1106 as judges. Arena-Hard results are reported as win rates against the baseline model. 

\textbf{Baselines.} We compare our~\methodabbrv with a comprehensive set of offline preference optimization methods:
 \begin{inparaenum}[(i)]
\item \textbf{RRHF}~\citep{yuan2023rrhf} and \textbf{SLiC-HF}~\citep{Zhao2023SLiCHFSL} rank losses using length-normalized and direct log-likelihood with an SFT objective, respectively;
\item \textbf{DPO}~\citep{rafailov2023direct} refers to the original direct preference optimization approach;
    \item \textbf{IPO}~\citep{Azar2023AGT} commits to preventing potential overfitting problems in DPO;
    \item \textbf{CPO}~\citep{xu2024contrastive} employs direct likelihood as a reward and trains in conjunction with a behavior cloning objective for winning responses;
    \item \textbf{KTO}~\citep{Ethayarajh2024KTOMA} eliminates the need for pair-wise preference datasets;
    \item \textbf{ORPO}~\citep{Hong2024ORPOMP} introduces a reference-model-free odd ratio term to penalize undesired generation styles;
    \item \textbf{R-DPO}~\citep{Park2024DisentanglingLF} incorporates additional regularization to prevent length exploitation;
    \item \textbf{SimPO}~\citep{meng2024simpo} proposes reference-free rewards that incorporate length normalization and target reward margins between winning and losing responses.  
\end{inparaenum}


\looseness=-1
\textbf{Results.} \Cref{tab:main_table} shows the primary evaluation outcomes of Mistral-7B and Llama-3-8B in both Base and Instruct configurations. While all methods effectively boost the performance over the SFT model, our proposed method, \methodabbrv, achieves the highest score on 16 out of 20 metrics while being competitive on the rest, which highlights the efficacy and robustness of our technique. Quantitatively, \methodabbrv achieves an average improvement of 1.4\% over the strongest baseline SimPO on the AlpacaEval 2 length-controlled win rate metric. Interestingly, in the Llama-3-8B Base setup, we are able to obtain a remarkable enhancement of 4.8\% (28.2\% against 23.4\%) in win rate compared with SimPO on Arena-Hard, further demonstrating the superiority of~\methodabbrv. The results consistently advocate~\methodabbrv as an effective and broadly applicable LLM alignment approach across a wide suite of models and tasks. The results in~\Cref{tab:openllm} further verifies the superiority of our~\methodabbrv, where~\methodabbrv outperforms SimPO consistently with both models (Llama3-8B-Instruct and Mistral-7B-Base) by a significant margin, leading to ~8\% relative improvement in average score. Notably,~\methodabbrv enhances the score of Mistral-7B-Base on the MUSR benchmark from 8.03 to 12.73, and Llama3-8B-Instruct on the IFEval benchmark from 22.66 to 24.04. The results further strengthen our findings of~\methodabbrv as a general and effective preference optimization approach.



\section{Conclusion}

In this paper, we introduce~\methodabbrv, a generalized approach towards preference optimization through $f$-divergence minimization. Our key insight lies in viewing LLM alignment task as distribution matching between the optimal and parameterized policy with the objective defined with $f$-divergence. Our approach generalizes existing techniques to a broader family of alignment objectives, among which the variant of $\alpha$-PO has been demonstrated to achieve state-of-the-art performance on various challenging LLM benchmarks. Our work offers a novel theoretical perspective in understanding preference optimization while enjoying strong empirical performance and wide applicability across different models and tasks.

\section*{Acknowledgment}

We gratefully acknowledge the support of ARO (W911NF-21-1-0125), ONR (N00014-23-1-2159), and Chan Zuckerberg Biohub.
We also gratefully acknowledge the support of NSF under Nos. OAC-1835598 (CINES), CCF-1918940 (Expeditions), DMS-2327709 (IHBEM), IIS-2403318 (III); Stanford Data Applications Initiative, Wu Tsai Neurosciences Institute, Stanford Institute for Human-Centered AI, Chan Zuckerberg Initiative, Amazon, Genentech, GSK, Hitachi, SAP, and UCB.
Minkai Xu thanks the generous support of Sequoia Capital Stanford Graduate Fellowship.



\bibliography{ref}

\clearpage

\appendix

\onecolumn
\aistatstitle{Supplementary Materials for\\
$f$-PO: Generalizing Preference Optimization with\\ 
$f$-divergence Minimization}

\section{Proofs}
\label{app:sec:proof}

\subsection{Proof of \Cref{thm:optimal-pi}}
In this section, we present the formal proof of \Cref{thm:optimal-pi}. For clarity we first repeat the theorem below.
\setcounter{theorem}{0} 
\begin{theorem}
    Let $\hat{\pi}_\theta(y|x) \propto \pi_\theta(y|x)^{\beta}  \pref(y|x)^{(1-{\beta})}$ and $\hat{\pi}^* \propto \pref \exp{\left(r(x,y)\right)}$. We define our alignment objective as minimizing the following $f$-divergence:
    \begin{equation}
        \FDiv(\hat{\pi}_\theta\|\hat{\pi}^*) 
        =\expect_{\hat{\pi}^*} \left[f\left(\frac{\hat{\pi}_\theta}{\hat{\pi}^*}\right)\right]
        =\expect_{\pref} \left[\frac{\hat{\pi}^*}{\pref}f\left(\frac{\hat{\pi}_\theta}{\hat{\pi}^*}\right)\right].
    \end{equation}
    With unlimited model capacity and perfect optimization, the optimal policy $\pi^*_\theta$ satisfies that $\pi^*_\theta = \pi^*$.
\end{theorem}

To prove the theorem, we consider the following property of $f$-divergence.
\begin{lemma}
    $ \FDiv(p\|q) \geq 0$ and the global optimal of $\FDiv(p\|q)$ only achieves when $q=p$. 
    \label{lemma:f}
\end{lemma}
\begin{proof}
By the definition of a convex function, for any $t \geq 0, f(t) \geq f(1)+f^{\prime}(1)(t-1)$. Given that $f(1)=0$, the inequality simplifies to $f(t) \geq f^{\prime}(1)(t-1)$. Then we have
\begin{equation}
f\left(\frac{p(x)}{q(x)}\right) \geq f^{\prime}(1)\left(\frac{p(x)}{q(x)}-1\right).
\end{equation}
Taking expectation over $q$, we have that
\begin{align}
\FDiv(p\|q) &= \int q(x) f\left(\frac{p(x)}{q(x)}\right) d x \geq \int q(x) f^{\prime}(1)\left(\frac{p(x)}{q(x)}-1\right) d x ,\nonumber \\
&= f^{\prime}(1)\left(\int p(x) d x-\int q(x) d x\right)=f^{\prime}(1)(1-1)=0.
\end{align}
The equation is satisfied only when $f\left(\frac{p(x)}{q(x)}\right) = 0$, \emph{i.e.} $\frac{p(x)}{q(x)} = 1$. Then we finish the proof of Lemma~\ref{lemma:f}. 
\end{proof}
Now we get back to the definition of $\hat{\pi}_\theta$ and $\hat{\pi}^*$:
\begin{align}
\hat{\pi}_\theta(y \mid x) = \pi_\theta(y \mid x)^\beta \pi_{\mathrm{ref}}(y \mid x)^{(1-\beta)}/ Z_{\theta}(x),~~~~~~~~~ \hat{\pi}^* =  \pref(y|x) \exp{\left(r(x,y)\right)} / Z_r(x).
\end{align}
Here $Z_{\theta}(x)$ and $Z_r(x)$ are the corresponding partition functions. By Lemma~\ref{lemma:f} we know that the minimizer of our $f$-divergence objective is achieved only when $\hat{\pi}_{\theta^*} = \hat{\pi}^*$, which is equivalent to
\begin{align}
     \pi_{\theta^*}(y \mid x)^\beta \pi_{\mathrm{ref}}(y \mid x)^{(1-\beta)}/ Z_{\theta}(x) =  \pref(y|x) \exp{\left(r(x,y)\right)} / Z_r(x), \nonumber \\
    \Leftrightarrow \pi_{\theta^*}(y \mid x)^{\beta}/ Z_{\theta}(x) =  \pref^\beta(y|x) \exp{\left(r(x,y)\right)} / Z_r(x), \nonumber \\
    \Leftrightarrow \pi_{\theta^*}(y \mid x) =  \pref(y|x) \exp{\left(\beta^{-1}r(x,y)\right)}  (Z_{\theta}(x)/ Z_r(x))^{\frac{1}{\beta}}.
\end{align}
Eliminating the partition-related constant, we have:
\begin{equation}
    \pi_{\theta^*}(y \mid x) \propto  \pref(y|x) \exp{\left(\beta^{-1}r(x,y)\right)} = \frac{1}{Z(x)} \pref(y|x) \exp{\left(\beta^{-1}r(x,y)\right)}  = \pi^{*},
\end{equation}
which concludes the proof.

\subsection{Proof of \Cref{thm:divergence}}

\begin{theorem}
    With $K \rightarrow \infty$,  asymptotically we have
    \begin{equation}
        \gL_\textnormal{\methodabbrv}(\pi_\theta) = \expect_{x \sim \gD} \FDiv \left(\hat{\pi}_\theta(y|x)\|\hat{\pi}^*(y|x)\right),
    \end{equation}
    where $\hat{\pi}_\theta(y|x)$ and $\hat{\pi}^*(y|x)$ follows the definition in \Cref{thm:optimal-pi}.
\end{theorem}

\begin{proof}
    We first consider following equation$
    \sum_{j=1}^K e^{g_\theta\left(x, y_j\right)} / K = \mathbb{E}_{\pref(y \mid x)} e^{g_\theta\left(x, y\right)}$.
When $K \to \infty$, the above equation holds because the LHS is the unbiased Monte Carlo estimation of the expectation. Note that $g_\theta(x, y)=\log \hat{\pi}_\theta(y \mid x)-\log \pi_{\mathrm{ref}}(y \mid x)= \log \frac{\hat{\pi}_\theta(y \mid x)}{\pi_{\mathrm{ref}}(y \mid x)}$. 
Then we have the fact that:
\begin{equation}
    \sum_{j=1}^K e^{g_\theta\left(x, y_j\right)} = K \mathbb{E}_{\pref(y \mid x)}\frac{\hat{\pi}_\theta(y \mid x)}{\pi_{\mathrm{ref}}(y \mid x)} = K.
\label{eq:mc_g}
\end{equation}
Similarly, we obtain an expectation formulation of $\sum_{j=1}^K e^{r\left(x, y_j\right)}$ as:
\begin{equation}
    \sum_{j=1}^K e^{r\left(x, y_j\right)} = K \mathbb{E}_{\pref(y \mid x)} e^{r(x, y)} = K Z_r(x).
\label{eq:mc_e}
\end{equation}
Here $Z_r(x)$ is a partition function with respect to $r$ and $x$. 
Putting the above Eq.~\ref{eq:mc_e} and Eq.~\ref{eq:mc_g} to the Eq.~\ref{eq:loss_fpo_general}, we obtain
\begin{align}
  \gL_\textnormal{\methodabbrv} (\pi_\theta) &=\expect_{x \sim \gD} \expect_{y_{1:K} \sim \pref} \Bigg[ \quad \frac{e^{r(x, y_{i})}}{\sum_{j=1}^K e^{r(x, y_{j})}} f\left(\frac{e^{g_\theta(x, y_{i})}}{\sum_{j=1}^K e^{g_\theta(x, y_{j})}} \Big/ \frac{e^{r(x, y_{i})}}{\sum_{j=1}^K e^{r(x, y_{j})}} \right)\Bigg], \nonumber \\
  &= \expect_{x \sim \gD} \Bigg[ \expect_{y_{1:K} \sim \pref} \quad \frac{e^{r(x, y_{i})}}{K Z_r(x)} f\left(\frac{\frac{\hat{\pi}_\theta(y_i \mid x)}{\pi_{\mathrm{ref}}(y_i \mid x)}}{K} \Big/ \frac{e^{r(x, y_{i})}}{ K Z_r(x)} \right)\Bigg], \nonumber \\
&=\expect_{x \sim \gD} \Bigg[ \frac{1}{K} \sum_{i=1}^{K} \frac{\pref(y_i\mid x)e^{r(x, y_{i})}}{Z_r(x)} f(\hat{\pi}_\theta(y_i \mid x)\Big/ \frac{\pref(y_i\mid x)e^{r(x, y_{i})}}{Z_r(x)})\Bigg] .
\end{align}
Recalling that $\frac{\pref(y_i \mid x)e^{r(x, y_{i})}}{Z_r(x)} = \hat{\pi}^*(y \mid x)$ given the definition, we have that
\begin{align}
    \gL_\textnormal{\methodabbrv} (\pi_\theta)  =  \expect_{x \sim \gD}  \Bigg[\expect_{\hat{\pi}^*(y \mid x)} f(\frac{\hat{\pi}_\theta(y_i \mid x)}{\hat{\pi}^*(y \mid x)} )\Bigg] = \mathbb{E}_{x \sim \mathcal{D}} \mathbb{D}_f\left(\hat{\pi}_\theta(y \mid x) \| \hat{\pi}^*(y \mid x)\right),
\end{align}
which finishes the proof. 
\end{proof}

\subsection{Derivation for EXO as \methodabbrv with Forward KL}

In this section, we provide a brief derivation on how EXO~\citep{ji2024towards} corresponds to minimizing the forward KL divergence $\KLDiv(\pi_\theta\|\pi^*)$.
\begin{proof}
By realizing the function $f$ in \Cref{eq:fpo} as the generator function of forward KL, \textit{i.e.}, $u \log u$, we have,
\begin{equation}
\begin{aligned}
\small
    \gL_\textnormal{\methodabbrv-RKL}(\pi_\theta) = & \expect_{x \sim \gD} \expect_{\{ y_w, y_l\} \sim \pref} \Bigg[ 
     (1-\epsilon) \cdot \frac{ \sigmoid(g_\theta(x, y_w)-g_\theta(x, y_l)) }{1-\epsilon} \log \left(\frac{ \sigmoid(g_\theta(x, y_w)-g_\theta(x, y_l)) }{1-\epsilon}\right) \\
    & \quad + \epsilon \cdot \frac{ \sigmoid(g_\theta(x, y_l)-g_\theta(x, y_w)) }{\epsilon} \log \left(\frac{ \sigmoid(g_\theta(x, y_l)-g_\theta(x, y_w)) }{\epsilon}\right)\Bigg], \\
    = & \expect_{x \sim \gD} \expect_{\{y_w, y_l\} \sim \pref} \Bigg[ 
     { \sigmoid(g_\theta(x, y_w)-g_\theta(x, y_l)) } \log \left(\frac{ \sigmoid(g_\theta(x, y_w)-g_\theta(x, y_l)) }{1-\epsilon}\right) \\
    & \quad + { \sigmoid(g_\theta(x, y_l)-g_\theta(x, y_w)) } \log \left(\frac{ \sigmoid(g_\theta(x, y_l)-g_\theta(x, y_w)) }{\epsilon}\right)\Bigg], \\
\end{aligned}
\end{equation}
which recovers the original EXO objective Eq. (23) in \citet{ji2024towards}.
\end{proof}

  

\section{More Experiment Details}
\label{app:sec:experiment_details}

This section provides comprehensive details on the experimental setup for our studies presented in \Cref{exp:function_f_ablation} and \Cref{exp:simpo_table}. We adhere to the training and evaluation protocols outlined in EXO \citep{ji2024towards} and SimPO \citep{meng2024simpo} for our method, while utilizing reported hyperparameters and results for baseline comparisons. The following subsections elaborate on the specific hyperparameters, training procedures, and evaluation methodologies employed in our experiments.

\subsection{Experiment Details in \Cref{exp:function_f_ablation}}

In our experiments, we follow the training and evaluation hyperparameters outlined in EXO \citep{ji2024towards}, while for all baselines, we utilized their reported hyperparameters and results for comparison. Below, we detail the key hyperparameters employed in \cite{ji2024towards}, as well as those used for our method. 

\textbf{Hyperparameters.}
For $\text{\methodabbrv}_{\text{pref}}$, we set $\beta_r = 0.5$, $\beta_{\pi} = 0.5$, and $\alpha = 0.1$. For $\text{\methodabbrv}_{\text{rw}}$, we use $\beta_r = 0.1$, $\beta_{\pi} = 0.35$, $\alpha = 0.1$, and $K = 4$. The label smoothing hyperparameter $\epsilon$ in $\text{\methodabbrv}_{\text{pref}}$ is set to 1e-3. Both methods employ the Adam optimizer with a learning rate of 1e-6 and a batch size of 64, training for one epoch on each dataset, which is more than sufficient for convergence.

\textbf{Evaluation Details.} We sample 4 completions from the learned policy for each of 512 prompts from the test set across all datasets. For consistency, we maintain a sampling temperature of $t = 0.8$ during both training and inference. To calculate the win rate using the reward model, we compare all pairs of completions between the learned policy and base completions (either from the SFT policy or the chosen completion in the dataset).

For GPT-4 evaluations, we sample 100 prompts with 1 completion per prompt for each policy. To mitigate position bias, we evaluate each pair of generations twice, swapping their order. We use the concise prompt from \citep{rafailov2023direct} for summary quality assessment and a modified prompt for evaluating helpfulness in multi-turn dialogues. The evaluation prompts are detailed in \Cref{app:sec:prompt_details}.

\subsection{Experiment Details in \Cref{exp:simpo_table}}

Hyperparameter tuning is critical for maximizing the performance of preference optimization methods. In our experiments, we adhered to the general training hyperparameters outlined in SimPO \citep{meng2024simpo}, while for all baselines, we utilized their reported hyperparameters and results for comparison. Below, we detail the setup and key hyperparameters employed in \citet{meng2024simpo}, as well as those used for our method, \methodabbrv.

\textbf{Setup.}
In Base setup, models were fine-tuned on the \href{https://huggingface.co/datasets/HuggingFaceH4/ultrachat_200k}{UltraChat-200k} corpus~\citep{Ding2023EnhancingCL} with a learning rate of 2e-5, batches of 128 samples, and sequence length capped at 2048 tokens. The learning schedule followed a cosine curve with a 10\% warm-up phase over a single epoch. The Adam optimizer~\citep{kingma2014adam} was used for all model training. For the preference optimization dataset, the 2048 token limit was maintained and the same cosine learning rate schedule and warm-up strategy were used.  For Instruct setting, we use specifically constructed preference datasets
\href{https://huggingface.co/datasets/princeton-nlp/llama3-ultrafeedback-armorm}{princeton-nlp/llama3-ultrafeedback-armorm}
for Llama-3 
and \href{https://huggingface.co/datasets/princeton-nlp/mistral-instruct-ultrafeedback}{princeton-nlp/mistral-instruct-ultrafeedback} for Mistral.
These datasets are created using prompts from UltraFeedback, with chosen and rejected response pairs $(y_w, y_l)$ regenerated using the SFT models and annotated by reward models, to approximate an on-policy setting for the Instruct models. 

\textbf{Hyperparameters.}
 We used the following hyperparameters for each experimental scenario: For Mistral-Base, we set $\alpha$=0.99, $\beta$=2.0, $\gamma$=1.6, with a learning rate of 3e-7. For Mistral-Instruct, we used $\alpha$=0.925, $\beta$=2.5, $\gamma$=0.3, and a learning rate of 6e-7. In the Llama3-Base setting, we employed $\alpha$=0.99, $\beta$=2.0, $\gamma$=1.0, with a learning rate of 7e-7. Finally, for Llama3-Instruct, we set $\alpha$=0.99, $\beta$=2.6, $\gamma$=1.43, and used a learning rate of 1e-6.

\textbf{Evaluation Details.}
For AlpacaEval 2, we generated responses using a sampling strategy. We set the temperature to 0.7 for Mistral-Base (in line with \texttt{zephyr-7b-beta}),\footnote{\url{https://github.com/tatsu-lab/alpaca_eval/blob/main/src/alpaca_eval/models_configs/zephyr-7b-beta/configs.yaml}} 0.5 for Mistral-Instruct (following \texttt{Snorkel-Mistral-PairRM-DPO}), and 0.9 for both Llama3 variants. Arena-Hard saw us employ greedy decoding across the board. For MT-Bench, we adhered to the official guidelines, which prescribe different sampling temperatures for various categories.

\subsection{Computing Resources}

In this work, we ran all our training experiments on four 80GB A100 GPUs, and all inference process using one 80GB A100 GPUs for each model.



\section{Prompt Details}
\label{app:sec:prompt_details}
In this section, we provide the specific prompts used in our model evaluation, as described in \Cref{exp:function_f_ablation,exp:simpo_table}. 

\subsection{Dialogue Generation Evaluation Prompt}

For evaluating the helpfulness of generated dialogues, particularly in multi-turn settings, we use the following prompt structure:

\begin{table}[!h]
\centering
\caption{Prompt for GPT-4 evaluation on the Anthropic HH dataset. Texts in blue are placeholders to be substituted by the real data.}
\resizebox{\columnwidth}{!}{%
\fbox{
\begin{minipage}{\textwidth}
\small
\textbf{For the following dialogue history to a chatbot, which response is more helpful?}\\
\textbf{Dialogue history:} \textcolor{blue}{\textless dialogue history\textgreater}\\
\textbf{Response A:} \textcolor{blue}{\textless Response A\textgreater}\\
\textbf{Response B:} \textcolor{blue}{\textless Response B\textgreater}\\
\textbf{FIRST}, provide a one-sentence comparison of the two responses and explain which you feel is more helpful.\\
\textbf{SECOND}, on a new line, state only \textbf{"A"} or \textbf{"B"} to indicate which response is more helpful. Your response should use the format:\\
\textbf{Comparison:} \textcolor{blue}{\textless one-sentence comparison and explanation\textgreater}\\
\textbf{More helpful:} \textcolor{blue}{\textless "A" or "B"\textgreater}
\end{minipage}
}}
\end{table}

\subsection{Summarization Evaluation Prompt}

For assessing the quality of summaries, we employ the concise prompt structure from \citet{rafailov2023direct}:

\begin{table}[!h]
\centering
\caption{Prompt for GPT-4 evaluation on the TL;DR  dataset. Texts in blue are placeholders to be substituted by real data.}
\resizebox{\columnwidth}{!}{%
\fbox{
\begin{minipage}{\textwidth}
\small
\textbf{Which of the following summaries does a better job of summarizing the most important points in the given forum post, without including unimportant or irrelevant details? A good summary is both precise and concise.}\\
\textbf{Post:} \textcolor{blue}{\textless post\textgreater}\\
\textbf{Summary A:} \textcolor{blue}{\textless Summary A\textgreater}\\
\textbf{Summary B:} \textcolor{blue}{\textless Summary B\textgreater}\\
\textbf{FIRST}, provide a one-sentence comparison of the two summaries, explaining which you prefer and why.\\
\textbf{SECOND}, on a new line, state only \textbf{"A"} or \textbf{"B"} to indicate your choice. Your response should use the format:\\
\textbf{Comparison:} \textcolor{blue}{\textless one-sentence comparison and explanation\textgreater}\\
\textbf{Preferred:} \textcolor{blue}{\textless "A" or "B"\textgreater}
\end{minipage}
}}
\end{table}

\end{document}